\documentclass[letterpaper]{article}
\usepackage{aaai21}
\usepackage{aaai21}  
\usepackage{times}  
\usepackage{helvet} 
\usepackage{courier}  
\usepackage[hyphens]{url}  
\usepackage{graphicx} 
\usepackage{natbib}  
\usepackage{caption} 
\usepackage{graphicx}
\usepackage{amssymb}
\usepackage{amsmath}
\usepackage{amsthm}
\usepackage{algorithm}
\usepackage{algorithmicx}
\usepackage{algpseudocode}
\usepackage{appendix}
\usepackage{fancyhdr}

\newtheorem{theorem}{Theorem}
\newtheorem{definition}{Definition}

\usepackage[switch]{lineno} 

\title{}
\author{$^1$, Reuth Mirsky$^1$, \and Peter Stone$^{1,2}$
\\ $^1$ The University of Texas at Austin
\\ $^2$ Sony AI
\\ \{wmacke, reuth, pstone\}@cs.utexas.edu}

\title{Expected Value of Communication \\ for Planning in Ad Hoc Teamworks}
\author {
    William Macke,\textsuperscript{\rm 1}
    Reuth Mirsky, \textsuperscript{\rm 1}
    Peter Stone \textsuperscript{\rm 1,2} \\
}
\affiliations {
    \textsuperscript{\rm 1} The University of Texas at Austin \\
    \textsuperscript{\rm 2} Sony AI \\
    \{wmacke,reuth,pstone\}@cs.utexas.edu
}
 
\begin{document}
 
\maketitle

\begin{abstract}
\begin{quote}
    A desirable goal for autonomous agents is to be able to coordinate on the fly with previously unknown teammates.  Known as “ad hoc teamwork”, enabling such a capability has been receiving increasing attention in the research community. One of the central challenges in ad hoc teamwork is quickly recognizing the current plans of other agents and planning accordingly.  In this paper, we focus on the scenario in which teammates can communicate with one another, but only at a cost.  Thus, they must carefully balance plan recognition based on observations vs. that based on communication.
    
    This paper proposes a new metric for evaluating how similar are two policies that a teammate may be following - the Expected Divergence Point (\textsc{edp}).
    We then present a novel planning algorithm for ad hoc teamwork, determining which query to ask and planning accordingly. We demonstrate the effectiveness of this algorithm in a range of increasingly general communication in ad hoc teamwork problems.
    
        
\end{quote}
\end{abstract}

\section{Introduction}
Modern autonomous agents are often required to solve complex tasks in challenging settings, and to do so as part of a team. For example, service robots have been deployed in hospitals to assist medical teams in the recent pandemic outbreak. The coordination strategy of such robots cannot always be fixed a priori, as it may involve previously unmet teammates that can have a variety of behaviors. 
These robots will only be effective if they are able to work together with other teammates without the need for coordination strategies provided in advance~\cite{cakmak2012designing}.
This motivation is the basis for ad hoc teamwork, which is defined from the perspective of a single agent, the \emph{Ego Agent}\footnote{Also referred to as the \emph{Ad Hoc Agent}}, that needs to collaborate with \emph{teammates} without any pre-coordination~\cite{stone2010ad,albrecht2018autonomous}. Without pre-coordination, only very limited knowledge about any teammate, such as that they have limited rationality or what their potential goals are, is available.
An important capability of ad hoc teamwork is plan recognition of teammates, as their goals and plans can affect the goals and plans of the ego agent.
Inferring these goals is not trivial, as the teammates may not provide any information about their policies, and the execution of their plan might be ambiguous. Hence, it is up to the ego agent to disambiguate between potential goals.
The first contribution of this paper is a \textbf{metric to evaluate ambiguity between two potential teammate policies} by quantifying the number of steps a teammate will take (in expectation) until it executes an action that is consistent with only one of the two policies. 
We show how this \textsc{edp} metric can be computed in practice using a Bellman update.

In addition to applying a reasoning process to independently infer the goal of its teammate, the ego agent can also directly \textit{communicate} with that teammate to gain information faster than it would get by just observing. However, if such a communication channel is available, it can come with a cost, and the ego agent should appropriately decide when and what to communicate. 
For example, if the previously described medical robot can fetch different tools for a physician in a hospital, the physician would generally prefer to avoid the additional cognitive load of communicating with the robot, but may be willing to answer an occasional question so that it can be a better collaborator. We refer to this setting where the ego agent can leverage communication to collaborate with little to no knowledge about its teammate as \emph{Communication in Ad hoc Teamwork}, or \textsc{cat}~\cite{mirskypenny}.
The second contribution of this paper is using \textsc{edp} in a
\textbf{novel planning algorithm for ad hoc teamwork} that reasons about the value of querying and chooses when and what to query about in the presence of previously unmet teammates.

Lastly, this paper presents empirical results showing the performance of the new \textsc{edp}-based algorithm in these complex settings, showing that it outperforms existing heuristics in terms of total number of steps required for the team to reach its goal. Moreover, the additional computations required of the \textsc{edp}-based algorithm can mostly take place prior to the execution, and hence its online execution time does not differ significantly from simpler heuristics.


\section{Related Work}
There is a vast literature on reasoning about teammates with unknown goals~\cite{fern2007decision,albrecht2018autonomous} and on communication between artificial agents~\cite{cohen1997team,decker1987distributed,pynadath2002communicative}, but significantly less works discuss the intersection between the two, and almost no work in an ad hoc teamwork setting.
Goldman and Zilberstein \shortcite{goldman2004decentralized} formalized the problem of collaborative communicating agents as a decentralized POMDP with communication (DEC-POMDP-com). 
Communication in Ad-Hoc Teamwork (\textsc{cat}) is a related problem that shares some assumptions with DEC-POMDP-com: 
\begin{itemize}
    \item All teammates strive to be collaborative.
    \item The agents have a predefined communication protocol available that cannot be modified during execution.
    \item The policies of the ego agent's teammates are set and cannot be changed.  This assumption does not mean that agents cannot react to other agents and their actions, but rather that such reactions are consistent as determined by the set policy.
\end{itemize} 
However, these two problems make different assumptions that separate them. DEC-POMDP-com uses a single model jointly representing all plans, and this model is collaboratively controlled by multiple agents. \textsc{cat}, on the other hand, is set from the perspective of one agent with a limited knowledge about its teammates' policies (such as that all agents share the same goal and strive to be collaborative) and thus it cannot plan a priori how it might affect these teammates \cite{stone2013teaching,ravula2019ad}.

In recent years there have been some works that considered communication in ad hoc teamwork \cite{barrett2014communicating,chakraborty2017coordinated}. 
These works suggested algorithms for ad hoc teamwork, where teammates either share a common communication protocol, or can test the policies of the teammates on the fly (e.g. by probing). 
These works are situated in a very restrictive multi-agent setting, namely a multi-arm bandit, where each task consists of a single choice of which arm to pull. Another recent work on multi-agent sequential plans proposed an Inverse Reinforcement Learning technique to infer teammates' goals on the fly, but it assumes that no explicit communication is available \cite{wang2020too}.

With recent developments in deep learning, several works were proposed for a sub-area of multi-agent systems, where agents share information using learned communication protocols~\cite{hernandez2019survey,mordatch2018emergence,foerster2016learning}. These works make several assumptions not used in this work: that the agents can learn new communication skills, and that an agent can train with its teammates before execution. An exception to that is the work of Van Zoelen et al. \shortcite{van2020learning} where an agent learns to communicate both beliefs and goals, and applies this knowledge within human-agent teams. Their work differs from ours in the type of communication they allow.

Several other metrics have been proposed in the past to evaluate the ambiguity of a domain and a plan. Worst Case Distinctiveness (\textsc{wcd}) is defined as the longest prefix that any two plans for different goals share \cite{keren2014goal}. Expected Case Distinctiveness (\textsc{ecd}) weighs the length of a path to a goal by the probability of an agent choosing that goal and takes the sum of all the weighted path lengths \cite{wayllace2017new}. Both these works only evaluate the distinctiveness between two goals with specific assumptions about how an agent plans, while \textsc{edp} evaluate the \emph{expected} case distinctiveness for \emph{any} pair of policies, which may or may not be policies to achieve two different goals. 

A specific type of \textsc{cat} scenarios refers to tasks where a single agent reasons about the sequential plans of other agents, and can gain additional information by querying its teammates or by changing the environment \cite{mirsky2018sequential,mirskypenny,shvo2020active}. In this paper, we focus on a specific variant of this problem known as Sequential One-shot Multi-Agent Limited Inquiry \textsc{cat} scenario, or \textsc{somali cat} \cite{mirskypenny}. Consider the use case of a service robot that is stationed in a hospital, that mainly has to retrieve supplies for physicians or nurses, and has two main goals to balance: understanding the task-specific goals of its human teammates, and understanding when it is appropriate to ask questions over acting autonomously. As the name \textsc{somali cat} implies, this scenario includes several additional assumptions: The task is episodic, one-shot, and requires a sequence of actions rather than a single action (like in the multi-armed bandit case); the environment is static, deterministic, and fully observable; the teammate is assumed to plan optimally, given that it is unaware of the ego agent's plans or costs; and there is one communication channel, where the ego agent can query as an action, and if it does, the teammate will forgo its action to reply truthfully (the communication channel is noiseless). The definition of \textsc{edp} and the algorithm presented in this paper rely on this set of assumptions as well. While previous work by the authors in \textsc{somali cat} gave effective methods for determining when to ask a query~\cite{mirskypenny}, they did not provide a principled method for determining what to query. In this work, we extend previous methods with a principled technique for constructing a query.

\section{Background}
The notation and terminology we use in this paper is patterned closely after that of Albrecht and Stone~\shortcite{albrecht2017reasoning}, extended to reason about communication between agents. We define a \textsc{somali cat} problem as a tuple $C=\langle S, A_A, A_T, T, C \rangle$ where $S$ is the set of states, $A_A$ is the set of actions the ad-hoc agent can execute, $A_T$ is the set of actions the teammate can execute, $T$ is the transition function $T:S\times A_A \times A_T \times S \to [0,1]$, and $C: A_A \times A_T \rightarrow \mathbb{R}$ maps joint actions to a cost. Specifically, $A_A = O_A \cup Q$ consists of a set of actions $O_A$ that can change the state of the world, which we refer to as \emph{ontic} actions, as defined in Muise et al. \shortcite{muise2015planning}, and a set of potential \emph{queries} $Q$ it can ask. Similarly, $A_T = O_T \cup R$ is a set of ontic actions that the teammate can execute $O_T$ and the set of possible \emph{responses} $R$ to the ego agent's query. 
 Actions in $O_i$ can be selected independently from the decisions of other agents, but if the ego agent chooses to query the teammate at timestep $t$, then the action of the ego agent is some $q \in Q$ and the teammate's action must be $r \in R$. In this case, if the ego agent queries, both agents forego the opportunity to select an ontic action; the ego agent uses an action to query and the teammate uses an action to respond.
A \emph{policy $\pi$} for agent $i$ is a distribution over tuples of states and actions, such that $\pi_{i}(s,a)$ is the probability with which agent $A_i$ executes $a \in A_i$ in state $s \in S$. 
Policy $\pi_i$ induces a \emph{trajectory} $Tr_i=\langle s^0, a_i^1, s^1, a_i^2, s^2... \rangle$.
The cost of a joint policy execution is the accumulated cost of the constituent joint actions of the induced trajectories. 
For simplicity, in this work we assume that all ontic actions have a joint cost of 1. Queries and responses have different costs depending on their complexity.

Additional useful notation that will be used throughout this paper is the \emph{Uniformly-Random-Optimal Policy} (or \textsc{urop}) $\hat{\pi}_{g}$, that denotes the policy that considers all plans that arrive at goal $g$ with minimal cost, samples uniformly at random from the set of these plans, and then chooses the first action of the sampled plan. 

In order to ground $Q$ to a discrete, finite set of potential queries, we first need to define a \emph{goal} of the teammate $g$ as the set of states in $S$ such that a set of desired properties holds (e.g., both the physician and the robot are in the right room). Given this definition, we can use a concrete set of queries $Q$ of the format: ``Is your goal in $G$?" where $G$ is a subset of possible goals. This format of queries was chosen as it can replace any type of query that disambiguates between potential goals (e.g. ``Is your goal on the right?", ``What are your next k actions?"), as long as we know how to map the query to the goals it can disambiguate.

\section{Expected Divergence Point}
The first major contribution of this paper is a formal way to quantify the ambiguity between two policies, based on the number of steps we expect to observe before we see an action that can only be explained by a plan that can be executed when following one policy, but not by a plan that follows the other policy. Consider 2 policies $\pi_1,\pi_2$, 
and assume that policy $\pi_2$ was used to construct the trajectory $Tr_2$.  
\begin{definition}
The \textbf{divergence point (\textsc{dp})} from $\pi_1$, or the minimal point in time such that we know $Tr_2$ was not produced by $\pi_1$, is defined as follows:
\begin{equation*}
    dp(\pi_1 \mid Tr_2) = min\{t \mid \pi_1(s_{t-1}, a_2^t) = 0\}
\end{equation*}
\end{definition}

Figure~\ref{fig:edp} shows an example in which the teammate is in the light gray tile $(4,3)$ and is marked using the robot image. The goal of this agent can either be $g_1$ or $g_2$ (dark gray tiles), and the policies $\hat{\pi}_{g_1}$ and $\hat{\pi}_{g_2}$ are the \textsc{urop}s for  $g_1$ and $g_2$ respectively.
If an agent were to follow the path outlined in red starting at state $(4,3)$, then the $dp$ of this path with $\hat{\pi}_{g_1}$ would be 3 as the path diverges from $\hat{\pi}_{g_1}$ at timestep 3. 

\begin{figure}
    \centering
    \includegraphics[width=\linewidth]{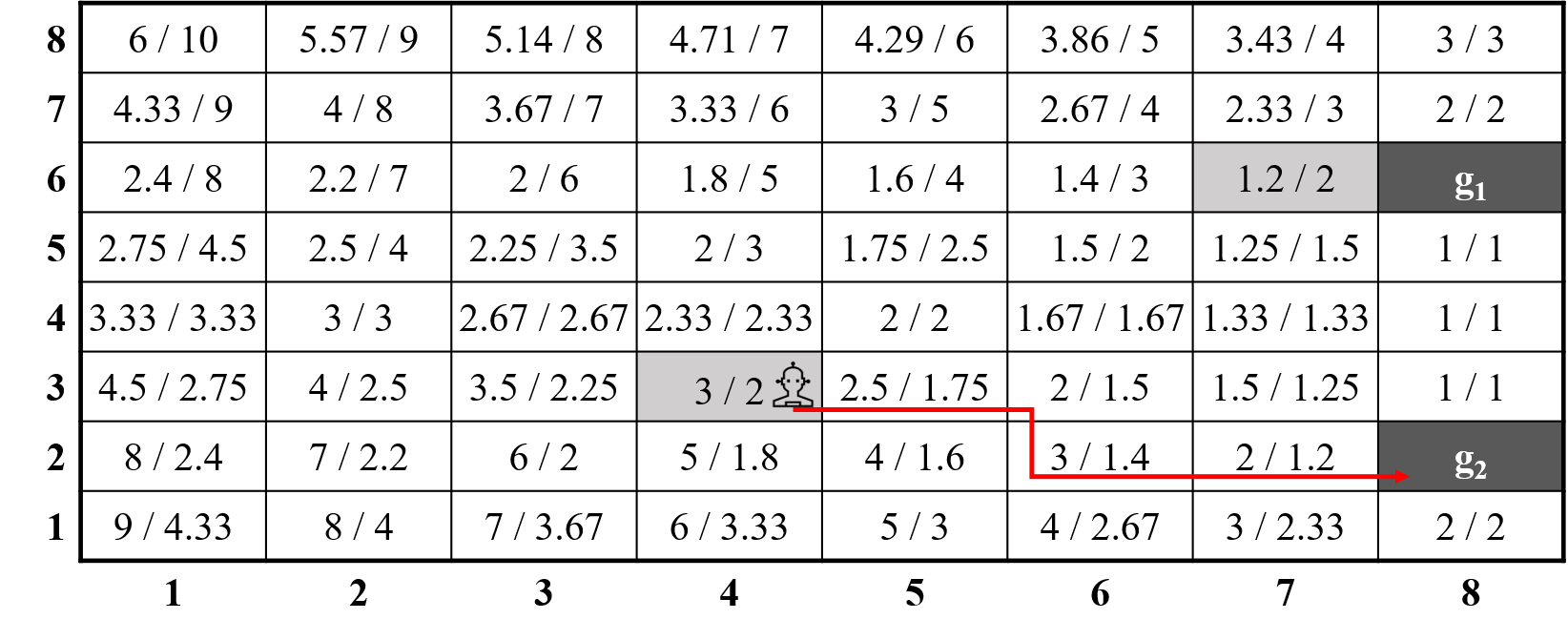}
    \vspace{-5mm}
    \caption{Various \textsc{edp} values in a grid world with two goals. Values are presented as $\textsc{edp}(s,\hat{\pi}_{g_1} | \hat{\pi}_{g_2})/\textsc{edp}(s, \hat{\pi}_{g_2} | \hat{\pi}_{g_1})$.}
    \label{fig:edp}
\end{figure}

To account for stochastic policies, where a policy will generally produce more than one potential trajectory from a state, we introduce the \emph{expected} divergence point that considers all possible trajectories constructed from $\pi_2$.
\begin{definition}
The \textbf{Expected Divergence Point (\textsc{edp})} of policy $\pi_1$ from $\pi_2$, starting from state $s$ is
\begin{equation*}
    \textsc{edp}(s, \pi_1 \mid \pi_2) = \mathbb{E}_{Tr_2}[dp(\pi_{1} \mid Tr_2)]
\end{equation*}
\end{definition}


\noindent Computing \textsc{edp} directly from this equation is non trivial. We therefore rewrite \textsc{edp} as the following Bellman update:
\begin{theorem}
\label{thm:bellman}
\textbf{The Bellman update for \textsc{edp}} of the policies $\pi_1, \pi_2$, starting from state $s$ is:
\begin{equation}
\scriptsize
\begin{split}
    \textsc{edp} & (s, \pi_1 |  \pi_2) = [1 - \sum\limits_{a\in A'(s)}\pi_2(s, a)] + \\ &
    \sum\limits_{a\in A'(s)}\pi_2(s,a) * \sum\limits_{s'\in S}T(s, a, s')[1 + \textsc{edp}(s', \pi_1 |  \pi_2)]
\end{split}
\normalsize
\label{eq:bellman}
\end{equation}
where $A'(s)=\{a\in A | \pi_1(s, a) > 0\}$.
\end{theorem}

\begin{proof}
We first define $P(dp(\pi_1 \mid Tr_2)=n)$ to be the probability of seeing $n$ steps before observing an action that $\pi_1$ will not take in state $s_{n-1}$, assuming that $Tr_2$ is some trajectory sampled from $\pi_2$. Then $\textsc{edp}$ can be written as:
\begin{equation}
\scriptsize
\label{eq:sum}
    \textsc{edp}(s, \pi_1 \mid \pi_2) = \sum\limits_{t=1}^\infty[P(dp(\pi_1 \mid Tr_2) = t)*t]
\normalsize
\end{equation}

Next, we compute this probability as the joint probability of the $k$-th action in $\pi_1$ being different from $a_2^k$ for all $k < n$ and $\pi_1(s_{n-1},a_1^n)=a_2^n$:

\begin{equation}
\label{eq:series}
\scriptsize
\begin{split}
    & P(dp(\pi_1 \mid Tr_2) = 1) = P(\pi_1(s_0, a_2^1) = 0), \\
    & P(dp(\pi_1 \mid Tr_2) = 2) = P(\pi_1(s_0, a_2^1) \neq 0)*P(\pi_1(s_1, a_2^2) = 0), \\
    & P(dp(\pi_1 \mid Tr_2) = 3) = P(\pi_1(s_0, a_2^1) \neq 0)*P(\pi_1(s_0, a_2^2) \neq 0) * \\ & \qquad \qquad \qquad \qquad \qquad ~~~~~~ P(\pi_1(s_1, a_2^3) = 0)
\end{split}
\normalsize
\end{equation}
and so on. Given these equations, we can rewrite the infinite summation in Equation \ref{eq:sum}:
\begin{equation*}
\scriptsize
\begin{split}
     \textsc{edp} & (s, \pi_1 | \pi_2) = \\
     & P(\pi_1(s_0, a_2^1) = 0) + \\
     & 2*P(\pi_1(s_0, a_2^1) \neq 0)*P(\pi_1(s_1, a_2^2) =0) + \\
     & 3*P(\pi_1(s_0, a_2^1) \neq 0)*P(\pi_1(s_0, a_2^2) \neq 0)*P(\pi_1(s_1, a_2^3) = 0) + \ldots
\end{split}
\normalsize
\end{equation*}

\noindent We can factor out a $P(\pi_1(s_0, a_1) \neq 0)$ from all of the first portion of the summation to get the following:
\begin{equation*}
\scriptsize
\begin{split}
    \textsc{edp} & (s, \pi_1 | \pi_2) = \\
    & P(\pi_1(s_0, a_2^1) = 0) + P(\pi_1(s_0, a_2^1) \neq 0) * \\
    & [2*P(\pi_1, a_2^2 = 0) + 3*(\pi_1(s_0, a_2^2) \neq 0)*P(\pi_1(s_1, a_2^3) = 0) + \ldots]
\end{split}
\normalsize
\end{equation*}
In reverse to the transition from Equation~\ref{eq:sum} to Equation~\ref{eq:series} and by using Bayes rule, the bracketed portion can be compiled back into an infinite sum: 
\begin{equation*}
\small
\begin{split}
    \textsc{edp} & (s, \pi_1 | \pi_2) = \\
    & P(\pi_1(s_0, a^1_2) = 0) + P(\pi_1(s_0, a^1_2) \neq 0) * \\ & \sum\limits_{t=1}^\infty[(t+1) * P(dp(\pi_1 \mid Tr_2) = t+1 \mid \pi_1(s_0, a^1_2) \neq 0)]
\end{split}
\normalsize
\end{equation*}
We distribute the multiplication by $(t+1)$ inside the summation to arrive at the following:
\begin{equation*}
\small
\begin{split}
    \textsc{edp} & (s, \pi_1 | \pi_2) = \\
    & P(\pi_1(s_0,a_2^1) = 0) + P(\pi_1(s_0, a_2^1) \neq 0) * \\ & [\sum\limits_{t=1}^\infty(t * P(dp_{\pi_1}(Tr_2) = t+1 \mid \pi_1(s_0, a_2^1) \neq 0)) + \\ & \sum\limits_{t=1}^\infty (P(dp_{\pi_1}(Tr_2) = t+1 \mid \pi_1(s_0, a_2^1) \neq 0))]
\end{split}
\normalsize
\end{equation*}

\noindent The second summation simplifies to 1, while the first is equivalent to the \textsc{edp} at the following state. Using this knowledge 
we derive the following formula for \textsc{edp}:
\begin{equation*}
\small
\begin{split}
    \textsc{edp} & (s, \pi_1 |  \pi_2) = \\
    & P(\pi_1(s_0, a_2^1) = 0) + P(\pi_1(s_0, a^2_1) \neq 0) * \\ &
    [\sum\limits_{s'\in S}T(s,a_2^0,s')*(1+\textsc{edp}(s', \pi_1 | \pi_2))]
\end{split}
\normalsize
\end{equation*}

Remember that the term $P(\pi_1(s_0, a_2^1) \neq 0)$ is the probability of sampling a first action in $Tr_2$ that will be the same as the action taken in $s_0$ according to $\pi_1$. 
Consider the set $A'(s) = \{a\in A | \pi_1(s,a) > 0\}$, and note that $P(\pi_1(s_0, a_2^1) \neq 0)$ is the same as $\sum\limits_{a\in A'(s_0)}\pi_2(s_0, a)$.
We therefore arrive at the Bellman update from Equation~\ref{eq:bellman}.
\end{proof}

There are a few things to note regarding \textsc{edp}. First, it is not a symmetric relation: $\textsc{edp}(s, \pi_1 | \pi_2)$ does not necessarily equal $\textsc{edp}(s, \pi_2 | \pi_1)$. For example, Figure~\ref{fig:edp} presents the value of $\textsc{edp}(s, \hat{\pi}_{g_1} | \hat{\pi}_{g_2})$ and the value of $\textsc{edp}(s, \hat{\pi}_{g_2} | \hat{\pi}_{g_1})$ side by side for each tile. In addition, neither of the $\textsc{edp}$ values necessarily equals the $\textsc{ecd}$ of this domain, as presented by Wayllace et al. \shortcite{wayllace2017new}, where $\textsc{ecd}$ estimates the probability of taking action $a$ in state $s$ as the normalized weighted probability of all goals for which $a$ is optimal from $s$. 
 Consider tile $(7,6)$ which is colored in light gray. Assuming uniform probability over goals, the $\textsc{ecd}$ when the teammate is in this state is $1.67$, a value distinct from both the \textsc{edp} values in this state, as well as from their average.
 
 Next, we show how \textsc{edp} can be used to compute the timesteps in which the ego agent is expected to benefit from querying, and then how this information can be used by the ego agent for planning when and what to query.
 
 
\section{Expected Zone of Querying}
Previous work provided a means to reason about when to act in the environment and when to query, in the form of three different reasoning zones general to all \textsc{somali cat} domains~\cite{mirskypenny}. The zones that were defined with respect to querying are: 
\begin{description} 
\item [Zone of Information ($Z_I$)] given the location of the teammate and two of its possible goals $g_1,g_2$, $Z_I$ is the interval from the beginning of the plan up to the worst case distinctiveness of $g_1,g_2$ for recognizing the teammate's goal, as defined by Keren et al. \shortcite{keren2014goal}.


Intuitively, these are the timesteps where the ego agent might gain information by querying about these goals.

\item [Zone of Branching ($Z_B$)] given the location of the ego agent and two possible goals of the teammate, $Z_B$ is the interval from the worst case distinctiveness of $g_1,g_2$ for recognizing the ego agent's goal and up to the end of the episode.
Intuitively, these are the timesteps where the ego agent might take a non-optimal action if it could not disambiguate between the two goals of its teammate.


\item [Zone of Querying ($Z_Q$)] given the locations of the ego agent, the teammate, and two possible goals of the teammate, $Z_Q$ is the intersection of $Z_I$ and $Z_B$ for these goals, where there may be a positive value in querying about the two goals instead of acting. Intuitively, these are the timesteps where the ego agent cannot distinguish between the two possible goals of the teammate and it should take different actions in each case\footnote{For the rest of the paper, when it is clear from context, we omit the state of the agents from the use of zones for brevity.}.
\end{description}

\noindent  Consider the running example from Figure~\ref{fig:edp}, and assume that this grid represents a maximal coverage task, where the agents need to go to different goals in order to cover as many goals as possible~\cite{pita2008deployed}. Consider an ego agent with the aim to go to a different goal from its teammate which is in tile $(4,3)$. The Zone of Information which depends only on the behavior of the teammate, is $Z_I(\{g_1,g_2\})=\{t \mid t \leq 5\}$, as 5 is the maximum number of steps that the teammate can take before its goal is disambiguated (e.g., if the agent moves east 4 times only its fifth action will disambiguate its goal). If the ego agent is in tile $(5,4)$, then $Z_B(\{g_1,g_2\})=\{t \mid t \geq 4\}$ and hence $Z_Q(\{g_1,g_2\})=\{4,5\}$. However, if the ego agent is in tile $(2,4)$, then $Z_B(\{g_1,g_2\})=\{t \mid t \geq 7\}$, and hence in this case $Z_Q(\{g_1,g_2\})=\emptyset$.
These existing definitions of zones consider the \emph{worst case} for disambiguating goals. However, the worst case might be highly uncommon, and planning when to query accordingly can induce high costs. 
Using \textsc{edp}, we now have the requirements to compute the \emph{expected case} for disambiguation of goals. We introduce definitions for expected zones: 
\begin{definition}
The \textbf{Expected Zone of Information $eZ_I$} given the location of the teammate and two goals $g_1, g_2$, $eZ_I$ is the expected time period during which the teammate's plan is ambiguous between $g_1$ and $g_2$:
\begin{equation*}
    eZ_I(s, g_1 | g_2) = \{t| t \leq \textsc{edp}(s, \hat{\pi}_{g_1} | \hat{\pi}_{g_2})\}
\end{equation*}
where the policies $\hat{\pi}_{g_1},\hat{\pi}_{g_2}$ are the \textsc{urop}s of the teammate for goals $g_1$ and $g_2$ respectively.

\end{definition}

\noindent Intuitively, $eZ_I$ for two goals is the set of timesteps where we don't expect to see an action that can only be executed by $g_1$ but not by $g_2$. Similarly, $eZ_B$ for two goals is the set of all timesteps where we expect that the ego agent will take a non-optimal action if it does not have perfect knowledge about the teammate’s true goal.

\begin{definition}
The \textbf{Expected Zone of Branching} for goals $g_1, g_2$ is the expected time period where the ego agent can take an optimal action for $g_2$ without incurring additional cost if the teammate's true goal is $g_1$:
\begin{equation*}
    eZ_B(s, g_1 | g_2) = \{t | t \geq \textsc{edp}(s, \hat{\pi}_{g_2} | \hat{\pi}_{g_1})\}
\end{equation*}
where the policies $\hat{\pi}_{g_1}$ and $\hat{\pi}_{g_2}$ are the \textsc{urop}s of the \emph{ego agent} for goals $g_1$ and $g_2$ respectively.
\end{definition}

\begin{definition}
The \textbf{Expected Zone of Querying}  for goals $g_1, g_2$ is the time period where the ego agent is expected to benefit from querying:
\begin{equation*}
    eZ_Q(s, g_1 | g_2) = eZ_B(s, g_1 | g_2) \cap eZ_I(s, g_1 | g_2)
\end{equation*}
\end{definition}
In our empirical settings, we have control over the ego agent, and can therefore guarantee it will follow an optimal plan given its current knowledge about the teammate's goal. In this case, we can use the original definition of $Z_B$ instead of $eZ_B$, as they are equivalent in this setting.

\section{Using $eZ_Q$ for Planning}
\label{sec:alg}
Next, we present a planning algorithm for the ego agent that uses the value of a query to minimize the expected cost of the chosen plan. In this planning process, there are two main problems that need to be solved: determining when to query and what to query. Using the definition of $Z_Q$, we can easily determine when querying would certainly be redundant, and respectively, when a query might be useful. Once we know that a query might be beneficial, we can use $eZ_Q$ to determine what to query exactly. Notice that the conclusion might still be not to query at all. For example, consider the maximal coverage example from Figure~\ref{fig:edp}, where the teammate is in tile $(4,3)$, and the ego agent is in tile $(8,4)$. As the teammate can choose to move east, there is still ambiguity between $g_1$ and $g_2$ and the ego agent must choose between moving north or south - so according to $Z_Q$ the ego agent should query. However, if it highly unlikely that the teammate would go east, the expected gain from querying decreases significantly. Thus, even though $Z_Q$ is not empty, it is not a good strategy for the ego agent to query. 
 
In this section, we discuss how to use the new $eZ_Q$ to compute the value of a specific query more accurately than proposed by previous approaches
This information will be used in an algorithm for the ego agent that chooses the best query. 
The first step is calculating the \textsc{edp} for each pair of goals $g_1$ and $g_2$ and their corresponding \textsc{urop}s $\hat{\pi}_{g_1}$ and $\hat{\pi}_{g_2}$. 
We use a modified version of the dynamic programming policy evaluation algorithm \cite{bellman1966dynamic} applied on the bellman update presented in Equation~\ref{eq:bellman} (additional details can be found in the Appendix. 
As this policy evaluation does not depend on the teammate's actions, it can be done offline prior to the plan execution, as presented in Figure~\ref{fig:edp}. Next, using these values, we can compute the value of a specific query.

\subsection{Computing the Value of a Query}

Given a \textsc{somali cat} domain $\langle S, A_A, A_T, T, C\rangle$, we want to 
quantify how much, in terms of expected plan cost, the ego agent can gain from asking a specific query.
Therefore, we define the value of a query as follows:
\begin{definition}
\textbf{Value of a Query} Let $q$ be a query with possible responses $R$, and $\mathcal{P}_g$ be the set of possible plans of both agents that arrive at goal $g$. Then the value of query $q$ is
\begin{equation}
    V(q) = \mathbb{E}_{p\in \mathcal{P}_g}[cost(p)] - \sum\limits_{r\in R}P(r)*\mathbb{E}_{p\in \mathcal{P}_g}[cost(p) | r]
\label{eq:value}
\end{equation}
\end{definition}

\noindent Computing the expected cost of this set of plans is non-trivial. We therefore define the following concept:
\begin{definition}
The \emph{Marginal Cost} of a plan $p$ in a \textsc{somali cat} domain with a goal $g$ ($MC_g(p)$) is the difference between the cost of $p$ and the cost of a minimal-cost plan that arrives at $g$.
\label{def:mc}
\end{definition}

\noindent We can replace $cost(p)$ in Equation~\ref{eq:value} with $MC_g(p)$ to yield an equivalent formula that is easier to compute:
\begin{equation}
\small
\begin{split}
    V & (q) = \\ 
    & (\mathbb{E}_{p\in \mathcal{P}_g}[cost(p)] - c^*_g) - (\sum\limits_{r\in R}P(r)*\mathbb{E}_{p\in \mathcal{P}_g}[cost(p) | r] - c^*_g) =\\ &
    \mathbb{E}_{p\in \mathcal{P}_g}[MC_g(p)] - \sum\limits_{r\in R}P(r)*\mathbb{E}_{p\in \mathcal{P}_g}[MC_g(p) | r]
    \end{split}
    \label{eq:v_q}
\end{equation}
where $c^*_g$ is the minimal cost of a plan that arrives at $g$.

In \textsc{somali cat},
$\mathbb{E}_{p\in\mathcal{P}_g}[MC_g(p)]$ at state $s$ is proportional to the number of timesteps in which the ego agent doesn't know which ontic actions are optimal, or formally:
\begin{equation*}
    \mathbb{E}_{p\in \mathcal{P}_g}[MC_g(p)]\propto |\bigcup\limits_{g'\in G}eZ_Q(s, g' | g)|
\end{equation*}

In addition, notice that computing $V(q)$ using Equation~\ref{eq:v_q} assumes that the goal of the teammate, $g$, is known. However, the ego agent does not know the true goal $g$ ahead of time, so we need to consider the expected value of a query for each possible goal of the teammate, or $\mathbb{E}_{g\in G}[V(q) | g]$.

\subsection{Choosing What to Query}


Our policy for determining whether or not to query at each timestep is shown in Algorithm~\ref{alg:query}. It takes as input the Zone of Querying for each pair of goals, $Z_Q$, the expected Zone of Querying $eZ_Q$, the current set of possible goals $G$, the ego agent’s current belief of the teammate’s goal $P$ and the current timestep $t$. First the algorithm checks if the current timestep is within a Zone of Querying of two goals or more. If not, then the agent knows of an optimal ontic action and no querying is required. Otherwise, we then find the best possible query given $eZ_Q$, and only ask this query if its value is greater than its cost.
\begin{algorithm}[t]
\caption{Query Policy}
\begin{algorithmic}
\Procedure{Query}{$Z_Q$ for each pair of goals, $eZ_Q$ for each pair of goals, $G$, $P$, current timestep $t$}
\If{$\exists g_1,g_2\in G \textit{  s.t.  } t\in Z_Q(g_1,g_2))$}
\State $query \gets ChooseQuery(G, eZ_Q, P)$
\If{$value(query) - cost(query) > 0$}
\State \Return $query$
\EndIf
\EndIf
\State \Return No Query
\EndProcedure
\end{algorithmic}
\label{alg:query}
\end{algorithm}

To optimize the expected value of a chosen query, we define a binary vector $\textbf{x}$ for each possible query, such that $x_i$ is 1 if and only if $g_i$ is included in that query. We then optimize for the difference between the value of a query above and its cost over these vectors using a genetic algorithm. We use a population size of $50$ and optimize for $100$ generations. Members are selected using tournament selection, and then combined using crossover. Each bit in the two resulting members is mutated with probability $0.001$.

\section{Experimental Setup}
\begin{figure}
    \centering
    \includegraphics[width=0.6\linewidth]{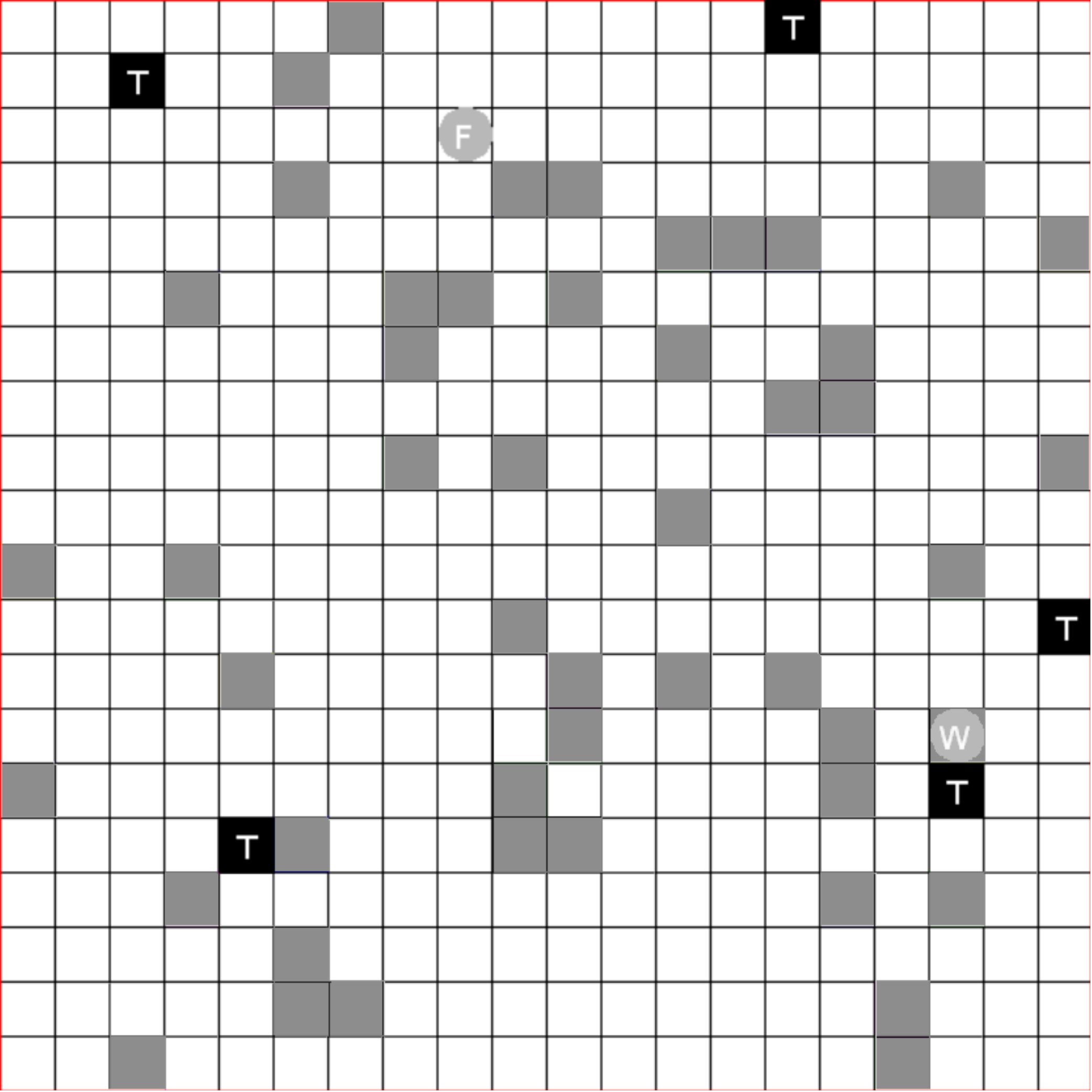}
    \caption{Example instance of the tool fetching domain. Workstations are shown by the grey boxes, toolboxes are shown by the black boxes with a T, the fetcher is shown as the gray circle labeled with an F, while the worker is shown as a gray circle labeled with a W.}
    \label{fig:domain_example}
\end{figure}

Previous work in \textsc{somali cat} introduced an experimental domain known as the tool fetching domain~\cite{mirskypenny}. This domain consists of an ego agent, the \emph{fetcher}, attempting to meet a teammate, the \emph{worker}, at some workstation with a tool. The worker needs a specific tool depending on which station is its goal, and the worker's goal and policy are unknown to the fetcher. It is the job of the fetcher to deduce the goal of the worker based on its actions, to pick up the correct tool from a toolbox and to bring it to the correct workstation. At each timestep, the agents can execute one ontic action each. In this work, the fetcher infers the worker's true goal by setting a goal's probability to zero and normalizing the belief distribution when observing a non-optimal action for that goal. Alternatively, the fetcher can query the worker with questions of the form ``Is your goal one of the stations $g_1,g_2...g_N$?", where $g_1, \ldots, g_N \subseteq G$ is a subset of all workstations, and the worker replies truthfully. All queries are assumed to have a cost identical to moving one step, regardless of the content of the query.
To show the benefits of our algorithm, we introduce 3 generalizations to the tool fetching domain that make the planning problem for the ego agent more complex to solve. 
    \paragraph{Multiple $Z_B$s} We allow multiple toolboxes in the domain. Each toolbox contains a unique set of tools, and only one tool for each station is present in a domain. Including this generalization means that each pair of goals may have different $Z_B$ values, which makes determining query timing and content more challenging. 
    \paragraph{Non-uniform Probability} We allow non-uniform probability distributions for assigning the worker a goal. For instance, goals may be assigned with probability corresponding to the Boltzmann distribution over the negative of their distances. This modification means that the worker is more likely to have goals that are closer to it. Including this generalization means that querying about certain goals may be more valuable than others, and the fetcher will have to consider this distribution when deciding what to query about.
    \paragraph{Cost Model} We allow for a more general cost model, where different queries have different costs. In particular, we consider a cost model where each query has some \emph{base cost} and an additional cost is added for each station asked, a \emph{per-station cost}. So for instance if queries have a base cost of $0.5$ and a per-station cost of $0.1$, then the query ``Is your goal station 1, 3 or 5?'' would have a cost of $0.5 + 3*0.1 = 0.8$. Including this generalization means that larger queries will cost more, and it may be more beneficial to ask smaller but less informative queries.
Results used a cost of $0.5$ when initiating a query and varied the per station cost of each query. 

We compare the performance of Algorithm \ref{alg:query} (\emph{$eZ_Q$ Query}) against two baseline ego agents: one agent that never queries but always waits when it is uncertain about which action to take (\emph{Never Query}). The other baseline is the algorithm introduced by Mirsky et al. \shortcite{mirskypenny}, where the ego agent chooses a random query once inside a $Z_Q$ (\emph{Baseline}). In addition, we extended \emph{Baseline} in two ways to make it choose queries in a more informed way. First by accounting for the changing cost of different queries, as well as for the probability distribution over the teammate's goals (\emph{BL:Cost+Prob})~\cite{MackeDMAP}. The second method involves creating a set of stations that each ontic action would be optimal for, and then querying about the set with the median size of these sets. Intuitively, this method first attempts to disambiguate which toolbox to reach and then attempts to disambiguate the worker's station  (\emph{BL:Toolbox}). All methods take a NOOP action if they are uncertain of the optimal action and do not query. Additional details about the setup and the strategies can be found in the Appendix.




\section{Results}
We ran experiments in a $20\times 20$ grid, with $50$ workstations and $5$ toolboxes. Locations of the stations, toolboxes, and agents in each domain instance are chosen randomly. All results are averaged over $100$ domain instances.

\begin{figure}[t]
    \centering
    \includegraphics[width=\linewidth]{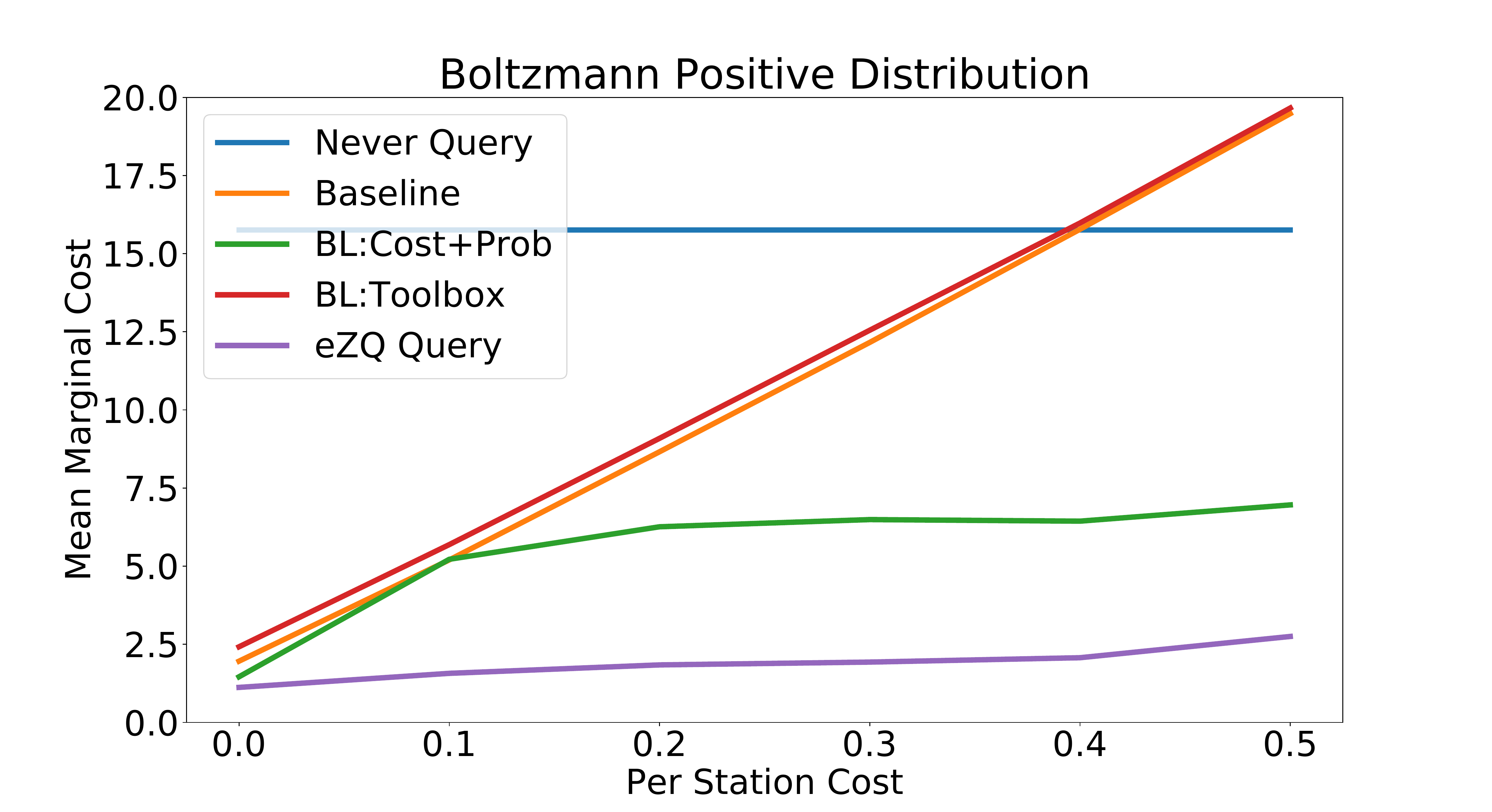}
    \vspace{-5mm}
    \caption{Per-station cost vs marginal cost of domain with Boltzmann distribution of the worker's distances to goals.}
    \label{fig:dist}
\end{figure}

\begin{figure}[t]
    \centering
    \includegraphics[width=\linewidth]{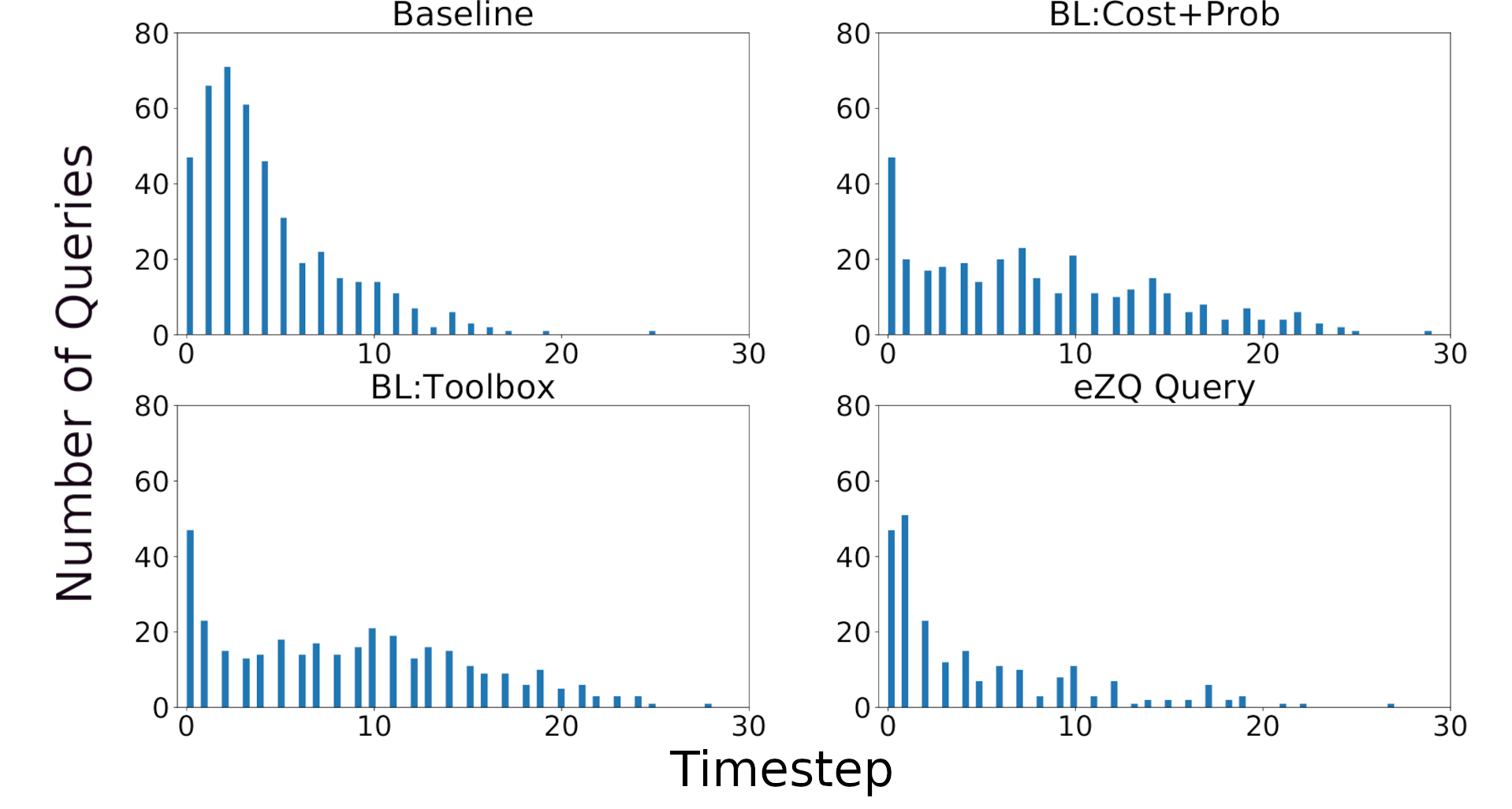}
    \vspace{-5mm}
    \caption{Histogram of number of queries (y-axis) executed per timestep (x-axis) by each method with the Boltzmann distribution of the worker's distances to goals and a per-station cost of 0.}
    \label{fig:hist}
\end{figure}

Figure~\ref{fig:domain_example} shows an example of such a tool-fetching domain. We now compare the new algorithm with previous work and the heuristics described above. We demonstrate that \emph{$eZ_Q$ Query} is able to effectively leverage the additional information from our generalizations to obtain a better performance over previous work and the suggested heuristics, in terms of marginal cost (Definition \ref{def:mc}). 

Figure~\ref{fig:dist} shows the marginal cost averaged over 100 domain instances with different per-station costs (x-axis) when the probability distribution of the goals of the worker is the Boltzmann distribution over the worker's distances to each goal. This distribution means that the worker is more likely to choose goals that are farther from it. \emph{$eZ_Q$ Query} performs similarly to the heuristics when the per-station cost is 0, but dramatically outperforms all other methods as this value increases.
Additional results showing \emph{$eZ_Q$ Query} under additional goal distributions can be seen in the Appendix.

We also provide an analysis of how the \emph{$eZ_Q$ Query} method is achieving better performance than other methods. Figure~\ref{fig:hist} shows a histogram of the number of queries executed per timestep by each method. As shown in the histogram, \emph{$eZ_Q$ Query} tends to ask more queries in the earlier timesteps compared to \emph{BL:Toolbox} and \emph{BL:Cost+Prob}. This increase is because \emph{$eZ_Q$ Query} is focused on learning the worker's true goal with minimal cost and is better able to leverage information about the probability distribution over goals to make informed queries, and therefore finds the correct goal more quickly than other methods, but also takes longer before it knows an optimal action.
In addition, we found that as the per-station cost increases from $0$ to $0.5$, the total number of queries executed by $eZ_Q$ query over all simulations decreases by $23\%$, showing that \emph{$eZ_Q$ Query} executes fewer queries when the cost is higher.



Finally, there is an increased computational cost of using \emph{$eZ_Q$ Query} over other approaches and the proposed heuristics. While calculating \textsc{ed}p is expensive and takes several orders of magnitude longer than the rest of the querying algorithm (several hours per domain on average), these values 
can be computed a priori regardless of the teammate's actions. As such, the following results are under the assumption that the \textsc{edp} computation is performed in advance, and the following time measurements do not include these offline computations. 
On average, all heuristic methods took $<0.23$ seconds to complete each simulation, while \emph{$eZ_Q$ Query} took on average $8.9$ seconds on an Ubuntu 16.04 LTS Intel core i7 2.5 GHz, with the genetic algorithm taking on average $6.1$ seconds to run. In practice, the increased time should not be a major detriment. If a robot is communicating with either a human or another robot, the major bottleneck is likely to be the communication channel (e.g. speech, network speed, decision making of the other agent) rather than this time. In addition, when using a genetic algorithm for optimization as we do in this paper, the \emph{$eZ_Q$ Query} computation should only grow in terms of $O(|G|^2log(|G|))$ with the number of goals (assuming that \textsc{edp} is precomputed ahead of time and that the number of members and generations do not grow with the number of goals).



\section{Discussion and Future Work}
In this paper, we investigated a new metric to quantify ambiguity of teammate policies in ad hoc teamwork settings, by estimating the expected divergence point between different policies a teammate might posses. We then utilized this metric to construct a new ad hoc agent that reasons both about \emph{when} it is beneficial to query, but also about \emph{what} is beneficial to query about in order to reduce the ambiguity about its teammate's goal. Our empirical results show that regardless of the goal-choosing policy of the worker and a varying query cost model, \emph{$eZ_Q$ Query} remains more effective than any of the other methods tested, and even when querying is almost never beneficial, it is still able to adapt and obtain performance that is consistently better than \emph{Never Query}.

The scope of this work is limited to \textsc{somali cat} problems. In addition, our current methods are designed to work in relatively simple environments with finite state spaces and a limited number of goals. However, the \textsc{edp} formalization opens up new avenues for investigating other 
complicated \textsc{somali cat} scenarios and other \textsc{cat} scenarios, such as those in which an agent can advise or share its beliefs with its teammates. We conjecture that the $eZ_Q$ algorithm can be modified relatively easily to address such challenges, as long as the ego agent remains the initiator of the communication. For instance, \textsc{edp} may be able to be calculated in domains with larger and continuous state spaces by leveraging more sophisticated RL techniques than the policy evaluation algorithm. It might be more challenging to extend this work to domains in which the teammate is the one to initiate the communication, as other works have investigated in the context of reinforcement learning agents \cite{torrey2013teaching,cui2018active}. Nonetheless, this work provides the means to investigate collaborations in ad hoc settings in new contexts, while presenting concrete solutions for planning in \textsc{somali cat} settings. 

\section*{Acknowledgements}
This work has taken place in the Learning Agents Research
Group (LARG) at UT Austin.  LARG research is supported in part by NSF
(CPS-1739964, IIS-1724157, NRI-1925082), ONR (N00014-18-2243), FLI
(RFP2-000), ARL, DARPA, Lockheed Martin, GM, and Bosch.  Peter Stone
serves as the Executive Director of Sony AI America and receives
financial compensation for this work.  The terms of this arrangement
have been reviewed and approved by the University of Texas at Austin
in accordance with its policy on objectivity in research.

\bibliography{aaai}
\bibliographystyle{aaai21}

\appendix
\section{Algorithm Description}
To calculate \textsc{edp} for two policies, we use a modified policy evaluation algorithm~\cite{bellman1966dynamic} using the Bellman update provided in Theorem \ref{thm:bellman}. Algorithm~\ref{alg:edp} shows this calculation: it performs sweeps of the whole state space until all values converge within some error bound $\epsilon$. The main loop updates the estimate of \textsc{edp} for each state using the Bellman Equation~\ref{eq:bellman} in the main paper.


\begin{algorithm}
\caption{Calculate EDP}
\begin{algorithmic}
\Procedure{EDP}{policy $\pi_1$, policy $\pi_2$, state space $S$, action space $A$, transition function $T$, desired accuracy $\epsilon$}
\State $EDP \gets \{0 | s \in S\}$
\State $err \gets \infty$
\While{$err > \epsilon$}
\State $err \gets 0$
\State $old \gets EDP$
\For{$s \in S$}
\State $A' \gets \{a\in A | \pi_1(s,a) > 0\}$
\State $EDP[s] \gets [1 - \sum\limits_{a\in A'}\pi_2(s,a)]$
\State $EDP[s] \gets EDP[s] + \sum\limits_{a \in A'}[\pi_2(s, a)] * \sum\limits_{s'\in S} T(s, a, s') * [1 + old[s']]$
\State $err \gets \max{(err, |EDP[s] - old[s]|)}$
\EndFor
\EndWhile
\EndProcedure
\end{algorithmic}
\label{alg:edp}
\end{algorithm}


\section{Experimental Setup}
\label{app:exp}

In this section we discuss additional details about the algorithms we tested and the experimental setup to allow reproducibility. In particular, we discuss further details about the \emph{BL:Cost+Prob} querying strategy and describe in detail how instances of the Tool Fetching Domain were generated.

\subsection{BL:Cost+Prob}
The goal of this method is to use a heuristic to determine what the ego agent should query about after we have determined that a query could be beneficial.
We first start by solving a simpler problem: determining what to query when the probability distribution over goals is uniform and every query has a uniform cost. Previous work would query about half the relevant goals, or in other words perform a binary search over the goals \cite{mirskypenny}. However reasoning more thoroughly regarding which goals to ask about can give even better performance. Consider the pairs of goals $G_B = \{(g_i,g_j) | t \in Z_B(g_i, g_j)\}$ where $t$ is the current time and $g_i, g_j \in G_B$ is a set of all $N$ possible goals that might still be the true goal of the worker. 
We can construct a binary vector $\vec{x}$ of length $N$ such that if $x_i$ is the $i$-th value in that vector, then station $g_i$ is included in the query if and only if $x_i=1$. A query that asks about a subset of goals $G'\subseteq G$ will disambiguate between the sets $G'$ and $G\setminus G'$. Therefore, to increase the information gained from the query, we want to split the pairs of goals $(g_i, g_j)$ as evenly as possible between $G'$ and $G\setminus G'$. We can write this maximization goal as follows:
\begin{equation}
    \max\sum\limits_{(g_i, g_j) \in G_B}(x_i \oplus x_j)
    \label{eq:o1}
\end{equation}
The term in the objective is 1 only when one $x$ is 0 and the other is 1. So for instance if the next action to reach goal $g_1$ must be the ontic action $o_1$, and the next action to reach goal $g_2$ must be the ontic action $o_2 \neq o_1$, then we need to eliminate one of these goals as a possibility before we know the next optimal ontic action.
This objective ensures that the query would ask about only one of these goals, and therefore about information that is the most relevant to which ontic action the ego agent should take next.

While the above is an improvement over randomly querying, it can give undesirable behavior when there is a non-uniform probability distribution over the worker's possible goals.
To reason about such circumstances, we modify the objective to weigh the goals by the ad-hoc agent's current belief state $P(g)$:
\begin{equation}
    \max\sum\limits_{(g_i, g_j) \in G_B}(P(g_i) + P(g_j))*(x_i \oplus x_j)
      \label{eq:o2}
\end{equation}
where $P(g_i)$ refers to the probability that the worker's goal is $g_i$. 
Intuitively, this new equation will prioritize goals that are more likely. This is desirable, since if a goal already has an extremely low probability, it likely won't be beneficial to query about it. We sum the two probabilities together since we want to consider these goals in the query if either of them is the true goal.

Finally, a complete model needs to incorporate different query cost models. Since we want to minimize the needed query cost as part of our objective, we add the negative cost of the query to our objective. Consequently, the final objective becomes
\begin{equation}
    \max\sum\limits_{(g_i,g_j) \in G_B}(x_i \oplus x_j)*(P(g_i) + P(g_j)) - \sum\limits_i(x_i*\textit{sc})
      \label{eq:o3}
\end{equation}
where \textit{sc} is the cost of including a station in a query. This objective now simultaneously attempts to maximize the probability that the ego agent will be able to act in the next timestep and minimize the cost of the query. Whenever the planning algorithm decides on querying, we solve Equation~\ref{eq:o3} using the Coin-OR Integer Program solver~\cite{john_forrest_2018_1317566}.

\subsection{Domain Instance Generation}
When generating a new domain instance, we first randomly select locations in the grid for the workstations and the toolboxes. Each location is chosen with uniform probability, where the only constraint is that two workstations/toolboxes cannot be in the same location. Then for each workstation, we assign its corresponding tool to one of the toolboxes with uniform probability, which means that different toolboxes might have a different number of tools in them. Finally, we choose a location for the worker and fetcher using uniform probability, and calculate the \textsc{edp} and \textsc{wcd} for each pair of stations/toolboxes as required for each of the algorithms. We serialize these results, and reuse them whenever running experiments. We generated 100 domain instances of 20x20 grids, with 50 workstations and 5 toolbox locations that were used in all experiments. Experiments were run using a custom OpenAI Gym Environment~\cite{openai_gym}.

\section{Additional Results}
\label{app:res}

\begin{figure}
    \centering
    \includegraphics[width=\linewidth]{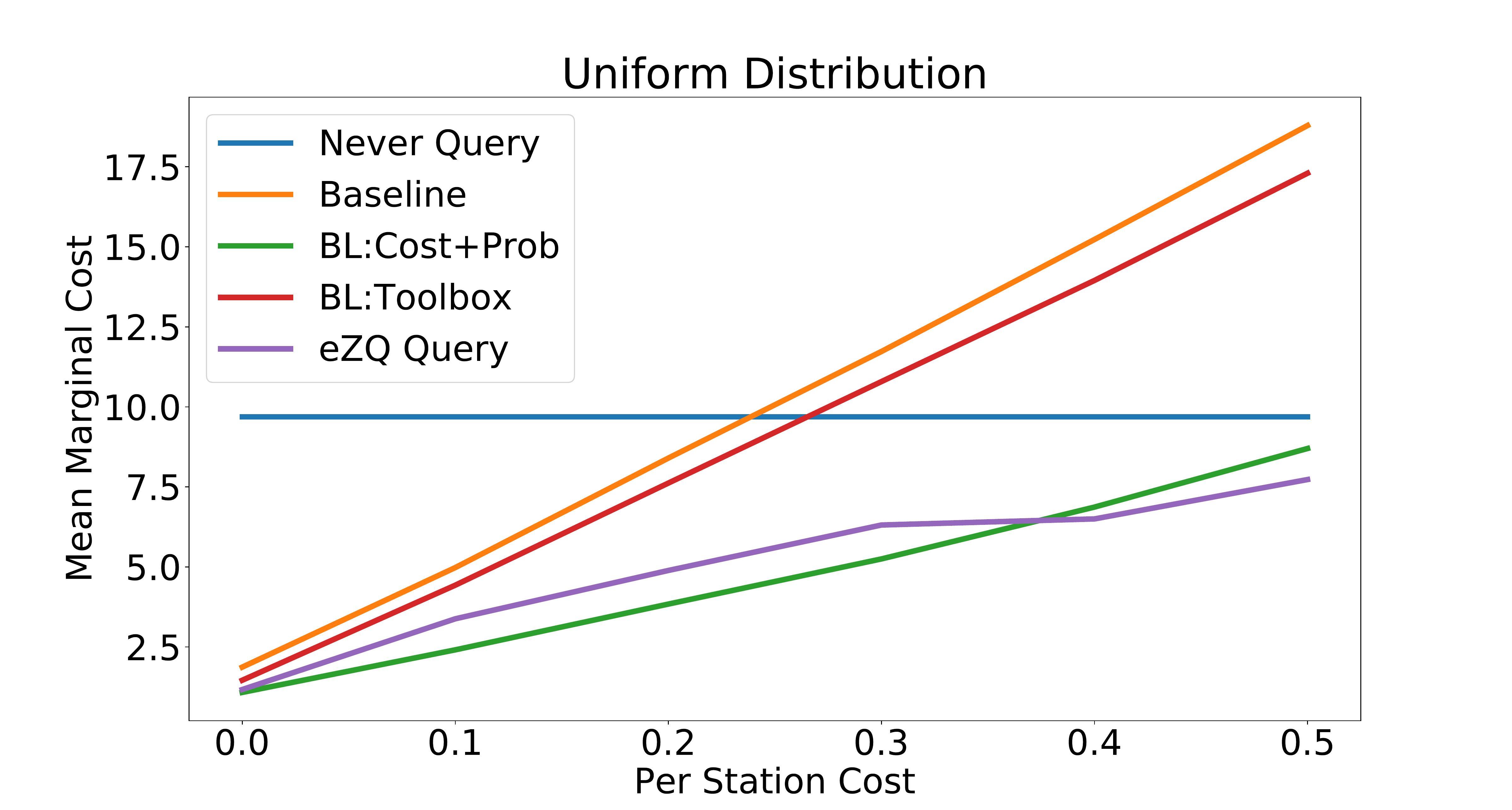}
    \caption{Per-station cost vs. marginal cost of domains with uniform distribution over the goals of the worker.}
    \label{fig:uniform}
\end{figure}

\begin{figure}
    \centering
    \includegraphics[width=\linewidth]{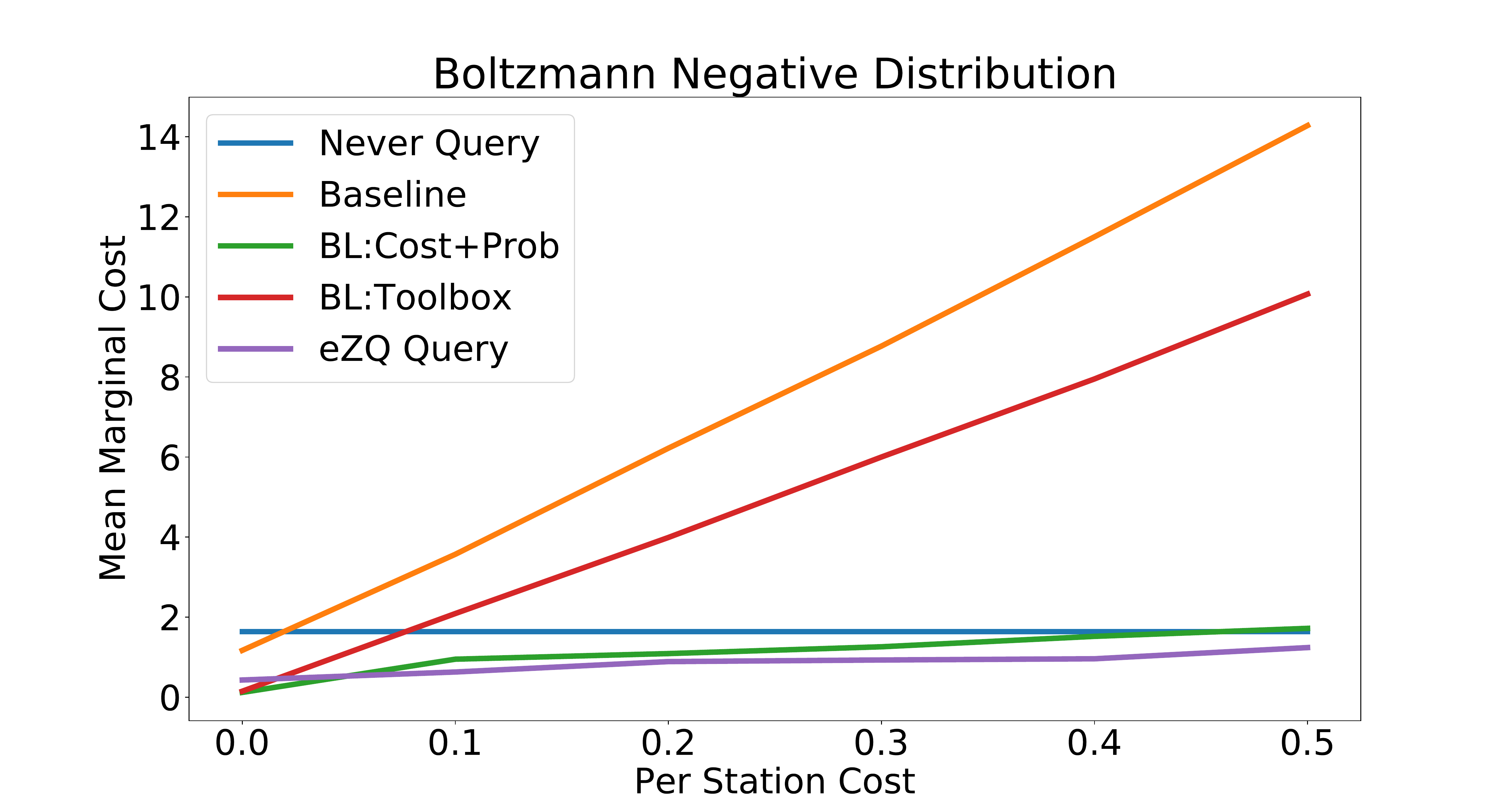}
    \caption{Per station cost vs marginal cost of domain with Boltzmann distribution of the negative of worker's distances to goals.}
    \label{fig:inv_dist}
\end{figure}

In our empirical results, we showed the performance of various querying strategies with the Boltzmann distribution over the worker's distances to goals. We present here the performance of the same strategies when using two additional distributions, a Uniform distribution and the Boltzmann distribution over the \emph{negative} of the worker's distances to goals. The Uniform distribution means that all goals are equally probable, and the Boltzmann distribution over the \emph{negative} of the worker's distances means that the worker is more likely to choose goals that are closer to it.
Figure~\ref{fig:uniform} shows the marginal cost with a uniform probability distribution. \emph{$eZ_Q$ Query} exhibits equivalent performance to other approaches when the per-station cost is $0$. As the per-station cost increases, the performance of the \emph{Baseline} methods dramatically worsen and none except \emph{BL:Cost+Prob} outperforms \emph{Never Query} when the per-station cost is large enough. By contrast, \emph{$eZ_Q$ Query} continues to perform well and is robust to the varying query costs.
Figure~\ref{fig:inv_dist} shows the average marginal cost with the Boltzmann distribution over the negative of the worker's distances to goals as the probability distribution. While the potential benefit from querying decreases under this distribution of potential goals, \emph{$eZ_Q$ Query} still outperforms all other querying strategies. \emph{BL:Toolbox} and \emph{BL:Cost+Prob} mildly outperform \emph{$eZ_Q$ Query} when the per-station cost is 0. However, \emph{$eZ_Q$ Query} is the only one to consistently outperform \emph{Never Query} when the per-station cost increases. \emph{Never Query} in this scenario performs extremely well since when the worker chooses a goal that's close to its current location, it reveals its true goal through its actions rather quickly. This dramatically reduces the overall need for querying in the first place. However, \emph{$eZ_Q$ Query} is still able to adapt and obtain performance that is consistently better than \emph{Never Query}.

With two exceptions, \emph{$eZ_Q$ Query} performs significantly better than all other approaches with a p value $<0.05$. The first exception is \emph{BL:Cost+Prob} with a uniform distribution over goals and a per-station cost $\leq 0.4$. The second exception is \emph{BL:Cost+Prob} with a Boltzmann distribution over the negative distances to goals and per station cost $\leq 0.2$.  We conjecture that our method fails to outperform \emph{BL:Cost+Prob} under these conditions because there was less information for our method to leverage. 
 For instance, when the distribution over goals is uniform and the per-station cost is 0, asking ``is your goal one of stations 1, 2 or 3?" is effectively the same query as ``is your goal one of stations 1, 2, 3 or 4?". This property makes it more difficult for our genetic algorithm to optimize, and can lead to worse performance.
\end{document}